\documentclass{article}

 \usepackage[final]{nips_2017}

\usepackage[utf8]{inputenc} 
\usepackage[T1]{fontenc}    
\usepackage{hyperref}       
\usepackage{url}            
\usepackage{booktabs}       
\usepackage{amsfonts}       
\usepackage{nicefrac}       
\usepackage{microtype}      

\usepackage{times}
\usepackage{enumitem}
\usepackage{color}

\usepackage{amsmath,amsfonts,amsthm,amssymb}
\usepackage{algpseudocode}
\usepackage{algorithm}
\usepackage{footnote}

\newcommand{\vct}{\boldsymbol }
\newcommand{\mat}{\mathbf}

\newcommand{\ud}{\mathrm d}

\newcommand{\argmax}{\mathrm{argmax}}
\newcommand{\poly}{\mathrm{poly}}
\newcommand{\rank}{\mathrm{rank}}

\newcommand{\diag}{\mathrm{diag}}

\newcommand{\tr}{\mathrm{tr}}

\newcommand{\Range}{\mathrm{Range}}
\newcommand{\Null}{\mathrm{Null}}

\newcommand{\nnz}{\mathrm{nnz}}

\def\mP{\mathbb{P}}

\newtheorem{remark}{Remark}[section]

\newtheorem{theorem}{Theorem}[section]
\newtheorem{lemma}{Lemma}[section]
\newtheorem{corollary}{Corollary}[section]

\newtheorem{definition}{Definition}[section]
\newtheorem{conjecture}{Conjecture}[section]

\newtheorem{question}{Question}

\title{On the Power of Truncated SVD for General High-rank Matrix Estimation Problems}

\author{
Simon S. Du\\
Carnegie Mellon University\\
\texttt{ssdu@cs.cmu.edu}
\And
Yining Wang\\
   Carnegie Mellon University\\
\texttt{yiningwa@cs.cmu.edu}\\
\And
Aarti Singh\\
Carnegie Mellon University\\
\texttt{aartisingh@cmu.edu}\\     
}

\begin{document}

\maketitle

\begin{abstract}
We show that given an estimate $\widehat{\mat A}$ that is close to a general high-rank positive semi-definite (PSD) matrix $\mat A$ in spectral norm
(i.e., $\|\widehat{\mat A}-\mat A\|_2 \leq \delta$),
the simple truncated Singular Value Decomposition of $\widehat{\mat A}$ produces a \emph{multiplicative} approximation of $\mat A$ in \emph{Frobenius} norm.
This observation leads to many interesting results on general high-rank matrix estimation problems:
\begin{enumerate}[leftmargin=*]
\item \emph{High-rank matrix completion}: we show that it is possible to recover a {general high-rank matrix} $\mat A$ up to $(1+\varepsilon)$ relative error
in Frobenius norm from partial observations, with sample complexity independent of the spectral gap of $\mat A$.
\item \emph{High-rank matrix denoising}: we design an algorithm that recovers a matrix $\mat A$ with error in Frobenius norm from its noise-perturbed observations,
without assuming $\mat A$ is exactly low-rank.
\item \emph{Low-dimensional approximation of high-dimensional covariance}: given $N$ i.i.d.~samples of dimension $n$ from $\mathcal N_n(\mat 0,\mat A)$,
we show that it is possible to approximate the covariance matrix $\mat A$ with relative error in Frobenius norm with $N\approx n$,
improving over classical covariance estimation results which requires $N\approx n^2$.
\end{enumerate}
\end{abstract}

\section{Introduction}

Let $\mat A$ be an unknown general high-rank $n\times n$ PSD data matrix that one wishes to estimate.
In many machine learning applications, though $\mat A$ is unknown, it is relatively easy to obtain a crude estimate $\widehat{\mat A}$
that is close to $\mat A$ in spectral norm (i.e., $\|\widehat{\mat A}-\mat A\|_2\leq\delta$).
For example, in matrix completion a simple procedure that fills all unobserved entries with 0 and re-scales observed entries produces an estimate 
that is consistent in spectral norm (assuming the matrix satisfies a spikeness condition, standard assumption in matrix completion literature).
In matrix de-noising, an observation that is corrupted by Gaussian noise is close to the underlying signal, because Gaussian noise is isotropic and has small spectral norm.
In covariance estimation, the sample covariance in low-dimensional settings is close to the population covariance in spectral norm under mild conditions \citep{bunea2015sample}.

However, in most such applications it is not sufficient to settle for a spectral norm approximation.
For example, in recommendation systems (an application of matrix completion) the zero-filled re-scaled rating matrix is close to the ground truth in spectral norm,
but it is an absurd estimator because most of the estimated ratings are zero.
It is hence mandatory to require a more stringent measure of performance.
One commonly used measure is the \emph{Frobenius norm} of the estimation error $\|\widehat{\mat A}-\mat A\|_F$,
which ensures that (on average) the estimate is close to the ground truth in an element-wise sense.
A spectral norm approximation $\widehat{\mat A}$ is in general \emph{not} a good estimate under Frobenius norm,
because in high-rank scenarios $\|\widehat{\mat A}-\mat A\|_F$ can be $\sqrt{n}$ times larger than $\|\widehat{\mat A}-\mat A\|_2$.

In this paper, we show that in many cases a powerful multiplicative low-rank approximation in Frobenius norm can be obtained
by applying a simple truncated SVD procedure on a crude, easy-to-find spectral norm approximate.
In particular, given the spectral norm approximation condition $\|\widehat{\mat A}-\mat A\|_2\leq\delta$, 
the top-$k$ SVD of $\widehat{\mat A}_k$ of $\widehat{\mat A}$ multiplicatively approximates $\mat A$ in Frobenius norm;
that is, $\|\widehat{\mat A}_k-\mat A\|_F\leq C(k,\delta,\sigma_{k+1}(\mat A))\|\mat A-\mat A_k\|_F$,
where $\mat A_k$ is the best rank-$k$ approximation of $\mat A$ in Frobenius and spectral norm.
To our knowledge, the best existing result under the assumption $\|\widehat{\mat A}-\mat A\|_2\leq \delta$ is due to \cite{achlioptas2007fast},
who showed that $\|\widehat{\mat A}_k-\mat A\|_F\leq \|\mat A-\mat A_k\|_F  + \sqrt{k}\delta + 2k^{1/4}\sqrt{\delta\|\mat A_k\|_F}$,
which depends on $\|\mat A_k\|_F$ and is not multiplicative in $\|\mat A-\mat A_k\|_F$.

Below we summarize applications in several matrix estimation problems.

\paragraph{High-rank matrix completion}
Matrix completion is the problem of (approximately) recovering a data matrix from very few observed entries.
It has wide applications in machine learning, especially in online recommendation systems.
Most existing work on matrix completion assumes the data matrix is \emph{exactly} low-rank \citep{candes2012exact,sun2016guaranteed,jain2013low}.
\citet{candes2010matrix,keshavan2010matrix} studied the problem of recovering a low-rank matrix corrupted by stochastic noise;
\citet{chen2016matrix} considered sparse column corruption.
All of the aforementioned work assumes that the ground-truth data matrix is exactly low-rank, which is rarely true in practice.

\citet{negahban2012restricted} derived minimax rates of estimation error when the spectrum of the data matrix lies in an $\ell_q$ ball.
\citet{zhang2015analysis,koltchinskii2011nuclear} derived oracle inequalities for general matrix completion; however their error bound has an additional $O(\sqrt{n})$
multiplicative factor.
These results also require solving computationally expensive nuclear-norm penalized optimization problems whereas our method only requires solving a single truncated singular value decomposition.
\citet{chatterjee2015matrix} also used the truncated SVD estimator for matrix completion.
However, his bound depends on the nuclear norm of the underlying matrix which may be $\sqrt{n}$ times larger than our result.
\cite{hardt2014fast} used a ``soft-deflation'' technique to remove condition number dependency in the sample complexity;
however, their error bound for general high-rank matrix completion is additive and depends on the ``consecutive'' spectral gap $(\sigma_k(\mat A)-\sigma_{k+1}(\mat A))$,
which can be small in practical settings \citep{balcan2016improved,anderson2015spectral}.
\citet{eriksson2012high} considered high-rank matrix completion with additional union-of-subspace structures.

In this paper, we show that if the $n\times n$ data matrix $\mat A$ satisfies $\mu_0$-spikeness condition,
\footnote{$n\|\mat A\|_{\max}\leq \mu_0\|\mat A\|_F$; see also Definition \ref{defn:spikeness}.}
then for any $\epsilon\in(0,1)$, the truncated SVD of zero-filled matrix $\widehat{\mat A}_k$ satisfies $\|\widehat{\mat A}_k-\mat A\|_F\leq (1+O(\epsilon))\|\mat A-\mat A_k\|_F$
if the sample complexity is lower bounded by $\Omega(\frac{n\max\{\epsilon^{-4},k^2\}\mu_0^2\|\mat A\|_F^2\log n}{\sigma_{k+1}(\mat A)^2})$
,which can be further simplified to $\Omega(\mu_0^2\max\{\epsilon^{-4},k^2\}\gamma_k(\mat A)^2\cdot nr_s(\mat A)\log n)$,
where $\gamma_k(\mat A)=\sigma_1(\mat A)/\sigma_{k+1}(\mat A)$ is the $k$th-order condition number
and $r_s(\mat A)=\|\mat A\|_F^2/\|\mat A\|_2^2 \leq \rank(\mat A)$ is the \emph{stable rank} of $\mat A$.
Compared to existing work, our error bound is multiplicative, gap-free, and the estimator is computationally efficient.
\footnote{We remark that our relative-error analysis does \emph{not}, however, apply to exact rank-$k$ matrix where $\sigma_{k+1}=0$.
This is because for exact rank-$k$ matrix a bound of the form $(1+O(\epsilon))\|\mat A-\mat A_k\|_F$ requires \emph{exact recovery} of $\mat A$,
which truncated SVD cannot achieve.
On the other hand, in the case of $\sigma_{k+1}=0$ a weaker additive-error bound is always applicable, as we show in Theorem~\ref{thm:mc}.
}

\paragraph{High-rank matrix de-noising}
Let $\widehat{\mat A}=\mat A+\mat E$ be a noisy observation of $\mat A$, where $\mat E$ is a PSD Gaussian noise matrix with zero mean and $\nu^2/n$ variance on each entry.
By simple concentration results we have $\|\widehat{\mat A}-\mat A\|_2 = \nu$ with high probability; however, $\widehat{\mat A}$ is 
in general not a good estimator of $\mat A$ in Frobenius norm when $\mat A$ is high-rank.
Specifically, $\|\widehat{\mat A}-\mat A\|_F$ can be as large as $\sqrt{n}\nu$.

Applying our main result, we show that if $\nu<\sigma_{k+1}(\mat A)$ for some $k\ll n$, then the top-$k$ SVD $\widehat{\mat A}_k$ of $\widehat{\mat A}$ satisfies
$\|\widehat{\mat A}_k-\mat A\|_F\leq (1+O(\sqrt{\nu/\sigma_{k+1}(\mat A)}))\|\mat A-\mat A_k\|_F + \sqrt{k}\nu$.
This suggests a form of bias-variance decomposition as larger rank threshold $k$ induces smaller bias $\|\mat A-\mat A_k\|_F$
but larger variance $k\nu^2$.
Our results generalize existing work on matrix de-noising \citep{donoho2014minimax,donoho2013phase,gavish2014optimal},
which focus primarily on exact low-rank $\mat A$.

\paragraph{Low-rank estimation of high-dimensional covariance}
The (Gaussian) covariance estimation problem asks to estimate an $n\times n$ PSD covariance matrix $\mat A$, either in spectral or Frobenius norm,
from $N$ i.i.d.~samples $X_1,\cdots,X_N\sim\mathcal N(\vct 0,\mat A)$.
The \emph{high-dimensional} regime of covariance estimation, in which $N\approx n$ or even $N\ll n$, 
has attracted enormous interest in the mathematical statistics literature \citep{cai2010optimal,cai2012optimal,cai2013optimal,cai2016estimating}.
While most existing work focus on sparse or banded covariance matrices, the setting where $\mat A$ has certain low-rank structure 
has seen rising interest recently \citep{bunea2015sample,kneip2011factor}.
In particular, \citet{bunea2015sample} shows that if $n=O(N^\beta)$ for some $\beta\geq 0$ then 
the sample covariance estimator $\widehat{\mat A}=\frac{1}{N}\sum_{i=1}^N{X_iX_i^\top}$ satisfies
\begin{equation}
\|\widehat{\mat A}-\mat A\|_F = O_\mP\left(\|\mat A\|_2 r_e(\mat A)\sqrt{\frac{\log N}{N}}\right),
\label{eq:bunea_frob}
\end{equation}
where $r_e(\mat A)=\tr(\mat A)/\|\mat A\|_2\leq \rank(\mat A)$ is the \emph{effective rank} of $\mat A$.
For high-rank matrices where $r_e(\mat A)\approx n$, Eq.~(\ref{eq:bunea_frob}) requires $N=\Omega(n^2\log n)$
to approximate $\mat A$ consistently in Frobenius norm.

In this paper we consider a reduced-rank estimator $\widehat{\mat A}_k$ and show that, if $\frac{r_e(\mat A)\max\{\epsilon^{-4},k^2\}\gamma_k(\mat A)^2\log N}{N}\leq c$
for some small universal constant $c>0$, then $\|\widehat{\mat A}_k-\mat A\|_F$
admits a relative Frobenius-norm error bound $(1+O(\epsilon))\|\mat A-\mat A_k\|_F$ with high probability.
Our result allows reasonable approximation of $\mat A$ in Frobenius norm under the regime of $N=\Omega(n\mathrm{poly}(k)\log n)$ if $\gamma_k = O\left(\poly\left(k\right)\right)$,
which is significantly more flexible than $N=\Omega(n^2\log n)$,
though the dependency of $\epsilon$ is worse than~\citep{bunea2015sample}.
The error bound is also agnostic in nature, making no assumption on the actual or effective rank of $\mat A$.



\paragraph{Notations}
For an $n\times n$ PSD matrix $\mat A$, denote $\mat A=\mat U\mat\Sigma\mat U^\top$ as its eigenvalue decomposition, where $\mat U$ is an orthogonal matrix
and $\mat\Sigma=\diag(\sigma_1,\cdots,\sigma_n)$ is a diagonal matrix, with eigenvalues sorted in descending order $\sigma_1\geq\sigma_2\geq\cdots\geq\sigma_n\geq 0$.
The spectral norm and Frobenius norm of $\mat A$ are defined as $\|\mat A\|_2=\sigma_1$ and $\|\mat A\|_F = \sqrt{\sigma_1^2+\cdots+\sigma_n^2}$, respectively.
Suppose $\vct u_1,\cdots,\vct u_n$ are eigenvectors associated with $\sigma_1,\cdots,\sigma_n$.
Define $\mat A_k=\sum_{i=1}^k{\sigma_i\vct u_i\vct u_i^\top} = \mat U_k\mat\Sigma_k\mat U_k^\top$, 
$\mat A_{n-k}=\sum_{i=k+1}^n{\sigma_i\vct u_i\vct u_i^\top}=\mat U_{n-k}\mat\Sigma_{n-k}\mat U_{n-k}^\top$
and $\mat A_{m_1:m_2} = \sum_{i=m_1+1}^{m_2}{\sigma_i\vct u_i\vct u_i^\top}=\mat U_{m_1:m_2}\mat\Sigma_{m_1:m_2}\mat U_{m_1:m_2}^\top$.
For a tall matrix $\mat U\in\mathbb R^{n\times k}$, we use $\mathcal U=\Range(\mat U)$ to denote the linear subspace
spanned by the columns of $\mat U$.
For two linear subspaces $\mathcal U$ and $\mathcal V$,
we write $\mathcal W=\mathcal U\oplus\mathcal V$ if $\mathcal U\cap\mathcal V=\{\vct 0\}$ and $\mathcal W=\{\vct u+\vct v: \vct u\in\mathcal U,\vct v\in\mathcal V\}$.
For a sequence of random variables $\{X_n\}_{n=1}^{\infty}$ and real-valued function $f:\mathbb N\to\mathbb R$,
we say
$X_n=O_\mP(f(n))$ if for any $\epsilon>0$, there exists $N\in\mathbb N$ and $C>0$ such that $\Pr[|X_n|\geq C\cdot |f(n)|]\leq\epsilon$ for all $n\geq N$.

\section{Multiplicative Frobenius-norm Approximation and Applications}\label{sec:main-result}

We first state our main result, which shows that truncated SVD on a weak estimator with small approximation error in spectral norm leads
to a strong estimator with \emph{multiplicative} Frobenius-norm error bound.
We remark that truncated SVD in general has time complexity \[O\left(\min\left\{n^2k,\nnz\left(\widetilde{\mat A}\right)+n\poly\left(k\right)\right\}\right),
\]
where $\nnz(\widetilde{\mat A})$ is the number of non-zero entries in $\widetilde{\mat A}$, and the time complexity is at most linear in matrix sizes when $k$ is small.
We refer readers to~\citep{allen2016even} for details.

\begin{theorem}
Suppose $\mat A$ is an $n\times n$ PSD matrix with eigenvalues $\sigma_1(\mat A)\geq\cdots\geq\sigma_n(\mat A)\geq 0$,
and a symmetric matrix $\widehat{\mat A}$ satisfies $\|\widehat{\mat A}-\mat A\|_2 \leq \delta = \epsilon^2\sigma_{k+1}(\mat A)$
for some $\epsilon\in(0,1/4]$.
Let $\mat A_k$ and $\widehat{\mat A}_k$ be the best rank-$k$ approximations of $\mat A$ and $\widehat{\mat A}$.
Then
\begin{equation}
\|\widehat{\mat A}_k - \mat A\|_F \leq (1+32\epsilon)\|\mat A-\mat A_k\|_F + 102\sqrt{2k}\epsilon^2\|\mat A-\mat A_k\|_2.
\label{eq:main}
\end{equation}
\label{thm:main}
\end{theorem}

\begin{remark}
Note when $\epsilon = O(1/\sqrt{k})$ we obtain an $\left(1+O\left(\epsilon\right)\right)$ error bound.
\label{rem:relative}
\end{remark}
\begin{remark}
This theorem only studies PSD matrices.
Using similar arguments in the proof, we believe similar results for general asymmetric matrices can be obtained as well.  
\end{remark}


To our knowledge, the best existing bound for $\|\widehat{\mat A}_k-\mat A\|_F$ assuming $\|\widehat{\mat A}-\mat A\|_2\leq\delta$ is due to
\citet{achlioptas2007fast}, who showed that
\begin{eqnarray}
\|\widehat{\mat A}_k-\mat A\|_F
&\leq& \|\mat A-\mat A_k\|_F + \|(\widehat{\mat A}-\mat A)_k\|_F + 2\sqrt{\|(\widehat{\mat A}-\mat A)_k\|_F\|\mat A_k\|_F}\nonumber\\
&\leq& \|\mat A-\mat A_k\|_F + \sqrt{k}\delta\|\mat A-\mat A_k\|_2 + 2k^{1/4}\sqrt{\delta}\sqrt{\|\mat A_k\|_F}.\label{eq:achlioptas}
\end{eqnarray}
Compared to Theorem \ref{thm:main}, Eq.~(\ref{eq:achlioptas}) is not relative because the third term $2k^{1/4}\sqrt{\|\mat A_k\|_F}$ depends on the $k$
\emph{largest} eigenvalues of $\mat A$, which could be much larger than the remainder term $\|\mat A-\mat A_k\|_F$.
In contrast, Theorem \ref{thm:main}, together with Remark \ref{rem:relative},
shows that $\|\widehat{\mat A}_k-\mat A\|_F$ could be upper bounded by a small factor multiplied with the remainder term $\|\mat A-\mat A_k\|_F$.

We also provide a gap-dependent version.
\begin{theorem}
Suppose $\mat A$ is an $n\times n$ PSD matrix with eigenvalues $\sigma_1(\mat A)\geq\cdots\geq\sigma_n(\mat A)\geq 0$,
and a symmetric matrix $\widehat{\mat A}$ satisfies $\|\widehat{\mat A}-\mat A\|_2 \leq \delta = \epsilon\left(\sigma_{k}(\mat A)-\sigma_{k+1}(\mat A)\right)$
for some $\epsilon\in(0,1/4]$.
Let $\mat A_k$ and $\widehat{\mat A}_k$ be the best rank-$k$ approximations of $\mat A$ and $\widehat{\mat A}$.
Then
\begin{equation}
\|\widehat{\mat A}_k - \mat A\|_F \leq \|\mat A-\mat A_k\|_F + 102\sqrt{2k}\epsilon\left(\sigma_{k}(\mat A)-\sigma_{k+1}(\mat A)\right).
\label{eq:main_with_gap}
\end{equation}
\label{thm:main_with_gap}
\end{theorem}
If $\mat A$ is an exact rank-$k$ matrix, Theorem~\ref{thm:main_with_gap} implies that truncated SVD gives an $\epsilon\sqrt{2k}\sigma_{k}$ error approximation in Frobenius norm, which has been established by many previous works~\citep{yi2016fast,tu2015low,wang2016unified}.

Before we proceed to the applications and proof of Theorem \ref{thm:main}, we first list several examples
of $\mat A$ with classical distribution of eigenvalues and discuss how Theorem \ref{thm:main} could be applied to obatin good Frobenius-norm approximations of $\mat A$.
We begin with the case where eigenvalues of $\mat A$ have a polynomial decay rate (i.e., power law).
Such matrices are ubiquitous in practice \citep{liu2015fast}.
\begin{corollary}[Power-law spectral decay]
Suppose $\|\hat{\mat A}-\mat A\|_2\leq\delta$ for some $\delta\in(0,1/2]$ and $\sigma_j(\mat A)=j^{-\beta}$ for some $\beta>1/2$.
Set $k= \lfloor\min\{C_1\delta^{-1/\beta}, n\}-1\rfloor$.
If $k\geq 1$ then
$$
\|\widehat{\mat A}_k-\mat A\|_F \leq C_1'\cdot\max\left\{\delta^{\frac{2\beta-1}{2\beta}}, n^{-\frac{2\beta-1}{2\beta}}\right\},
$$
where $C_1,C_1'>0$ are constants that only depend on $\beta$.
\label{cor:power-law}
\end{corollary}
We remark that the assumption $\sigma_j(\mat A)=j^{-\beta}$ implies that the eigenvalues
lie in an $\ell_q$ ball for $q=1/\beta$; that is, $\sum_{j=1}^n{\sigma_j(\mat A)^q} = O(1)$.
The error bound in Corollary \ref{cor:power-law} matches the minimax rate (derived by \citet{negahban2012restricted}) for matrix completion when the spectrum is constrained in an $\ell_q$ ball,
by replacing $\delta$ with $\sqrt{n/N}$ where $N$ is the number of observed entries.

Next, we consider the case where eigenvalues satisfy a faster decay rate.
\begin{corollary}[Exponential spectral decay]
Suppose $\|\hat{\mat A}-\mat A\|_2\leq\delta$ for some $\delta\in(0,e^{-16})$ and $\sigma_j(\mat A)=\exp\{-cj\}$ for some $c>0$.
Set $k=\lfloor\min\{c^{-1}\log(1/\delta)-c^{-1}\log\log(1/\delta), n\}-1\rfloor$.
If $k\geq 1$ then
$$
\|\widehat{\mat A}_k-\mat A\|_F \leq C_2'\cdot \max\left\{\delta\sqrt{\log(1/\delta)^3}, n^{1/2}\exp(-cn)\right\},
$$
where $C_2'>0$ is a constant that only depends on $c$.
\label{cor:exponential-decay}
\end{corollary}

Both corollaries are proved in the appendix.
The error bounds in both Corollaries \ref{cor:power-law} and \ref{cor:exponential-decay} are significantly better than the trivial estimate $\widehat{\mat A}$,
which satisfies $\|\widehat{\mat A}-\mat A\|_F \leq n^{1/2}\delta$.
We also remark that the bound in Corollary \ref{cor:power-law} cannot be obtained by a direct application of the weaker bound Eq.~(\ref{eq:achlioptas}), 
which yields a $\delta^{\frac{\beta}{2\beta-1}}$ bound.

We next state results that are consequences of Theorem \ref{thm:main} in several matrix estimation problems.

\subsection{High-rank Matrix Completion}

Suppose $\mat A$ is a high-rank $n\times n$ PSD matrix that satisfies $\mu_0$-spikeness condition defined as follows:
\begin{definition}[Spikeness condition]
An $n\times n$ PSD matrix $\mat A$ satisfies \emph{$\mu_0$-spikeness condition} if $n\|\mat A\|_{\max}\leq\mu_0\|\mat A\|_F$,
where $\|\mat A\|_{\max} = \max_{1\leq i,j\leq n}|\mat A_{ij}|$ is the max-norm of $\mat A$.
\label{defn:spikeness}
\end{definition}

Spikeness condition makes uniform sampling of matrix entries powerful in matrix completion problems.
If $\mat A$ is exactly low rank, the spikeness condition is implied by an upper bound on $\max_{1\leq i\leq n}{\|\vct e_i^\top\mat U_k\|_2}$,
which is the standard incoherence assumption on the top-$k$ space of $\mat A$ \citep{candes2012exact}.
For general high-rank $\mat A$, the spikeness condition is implied by a more restrictive incoherence condition that
imposes an upper bound on $\max_{1\leq i\leq n}\|\vct e_i^\top\mat U_{n-k}\|_2$ and $\|\mat A_{n-k}\|_{\max}$, which are assumptions adopted in \citep{hardt2014fast}.

Suppose $\widehat{\mat A}$ is a symmetric re-scaled zero-filled matrix of observed entries. That is, 
\begin{equation}
[\widehat{\mat A}]_{ij} = \left\{\begin{array}{ll}
\mat A_{ij}/p,& \text{with probability $p$};\\
0,& \text{with probability $1-p$;}\end{array}
\right.\;\;\;\;\;\;
\forall 1\leq i\leq j\leq n.
\label{eq:model}
\end{equation}
Here $p\in(0,1)$ is a parameter that controls the probability of observing a particular entry in $\mat A$,
corresponding to a sample complexity of $O(n^2p)$.
Note that both $\widehat{\mat A}$ and $\mat A$ are symmetric so we only specify the upper triangle of $\widehat{\mat A}$.
By a simple application of matrix Bernstein inequality \citep{mackey2014matrix},
one can show $\widehat{\mat A}$ is close to $\mat A$ in spectral norm when $\mat A$ satisfies $\mu_0$-spikeness.
Here we cite a lemma from \citep{hardt2014understanding} to formally establish this observation:
\begin{lemma}[Corollary of~\citep{hardt2014understanding}, Lemma A.3]
Under the model of Eq.~(\ref{eq:model}) and $\mu_0$-spikeness condition of $\mat A$, for $t\in(0,1)$ it holds with probability at least $1-t$ that
$$
\|\widehat{\mat A}-\mat A\|_2 \leq O\left(\max\left\{\sqrt{\frac{\mu_0^2\|\mat A\|_F^2\log(n/t)}{np}}, \frac{\mu_0\|\mat A\|_F\log(n/t)}{np}\right\}\right).
$$
\end{lemma}

Let $\widehat{\mat A}_k$ be the best rank-$k$ approximation of $\widehat{\mat A}$ in Frobenius/spectral norm.
Applying Theorem \ref{thm:main} and \ref{thm:main_with_gap} we obatin the following result:
\begin{theorem}
Fix $t\in(0,1)$. Then with probability $1-t$ we have
\begin{equation*}
\|\widehat{\mat A}_k-\mat A\|_F\leq O(\sqrt{k})\cdot \|\mat A-\mat A_k\|_F
\;\;\;\;\;\;\text{if}\;\;\;\;\; p = \Omega\left(\frac{\mu_0^2\|\mat A\|_F^2\log(n/t)}{n\sigma_{k+1}(\mat A)^2}\right).
\end{equation*}
Furthermore, for fixed $\epsilon\in(0,1/4]$, with probability $1-t$ we have
\begin{equation*}
\|\widehat{\mat A}_k-\mat A\|_F\leq \left(1+O(\epsilon)\right)\|\mat A-\mat A_k\|_F
\;\;\;\;\text{if}\;\;\; p = \Omega\left(\frac{\mu_0^2\max\{\epsilon^{-4},k^2\}\|\mat A\|_F^2\log(n/t)}{n\sigma_{k+1}(\mat A)^2}\right)
\end{equation*}
\begin{equation*}
\|\widehat{\mat A}_k-\mat A\|_F\leq \|\mat A-\mat A_k\|_F
+\epsilon\left(\sigma_k\left(\mat A\right) - \sigma_{k+1}\left(\mat A\right)\right)\;\;\text{if}\;\; p = \Omega\left(\frac{\mu_0^2k\|\mat A\|_F^2\log(n/t)}{n\epsilon^2\left(\sigma_{k}(\mat A) -\sigma_{k+1}(\mat A) \right)^2}\right).
\end{equation*}
\label{thm:mc}
\end{theorem}

As a remark, because $\mu_0\geq 1$ and $\|\mat A\|_F/\sigma_{k+1}(\mat A)\geq \sqrt{k}$ \emph{always} hold,  
the sample complexity is lower bounded by $\Omega(nk\log n)$,
the typical sample complexity in noiseless matrix completion.
In the case of high rank $\mat A$, the results in Theorem \ref{thm:mc} are the strongest when $\mat A$ has small \emph{stable rank} $r_s(\mat A) = \|\mat A\|_F^2/\|\mat A\|_2^2$
and the top-$k$ condition number $\gamma_k(\mat A)=\sigma_1(\mat A)/\sigma_{k+1}(\mat A)$ is not too large.
For example, if $\mat A$ has stable rank $r_s(\mat A)=r$ then
$\|\widehat{\mat A}_k-\mat A\|_F$ has an $O(\sqrt{k})$ multiplicative error bound with sample complexity $\Omega(\mu_0^2\gamma_k(\mat A)^2\cdot nr\log n)$;
or an $(1+O(\epsilon))$ relative error bound with sample complexity $\Omega(\mu_0^2\max\{\epsilon^{-4},k^2\}\gamma_k(\mat A)^2\cdot nr\log n)$.
Finally, when $\sigma_{k+1}(\mat A)$ is very small and the ``gap'' $\sigma_k(\mat A)-\sigma_{k+1}(\mat A)$ is large, 
a weaker additive-error bound is applicable with sample complexity independent of $\sigma_{k+1}(\mat A)^{-1}$.

Comparing with previous works, if` the gap $(1-\sigma_{k+1}/\sigma_k)$ is of order $\epsilon$, then sample complexity of\citep{hardt2014understanding} Theorem 1.2 and \citep{hardt2014fast} Theorem 1 scale with $1/\epsilon^7$. 
 Our result improves their results to the scaling of $1/\epsilon^4$ with a much simpler algorithm (truncated SVD). 

\subsection{High-rank matrix de-noising}

Let $\mat A$ be an $n\times n$ PSD signal matrix and $\mat E$ a symmetric random Gaussian matrix with zero mean and $\nu^2/n$ variance.
That is, $\mat E_{ij}\overset{i.i.d.}{\sim}\mathcal N(0,\nu^2/n)$ for $1\leq i\leq j\leq n$ and $\mat E_{ij}=\mat E_{ji}$.
Define $\widehat{\mat A}=\mat A+\mat E$.
The matrix de-noising problem is then to recover the signal matrix $\mat A$ from noisy observations $\widehat{\mat A}$.
We refer the readers to \citep{gavish2014optimal} for a list of references that shows the ubiquitous application of matrix de-noising in scientific fields.

It is well-known by concentration results of Gaussian random matrices, that $\|\widehat{\mat A}-\mat A\|_2=\|\mat E\|_2 = O_\mP(\nu)$.
Let $\widehat{\mat A}_k$ be the best rank-$k$ approximation of $\widehat{\mat A}$ in Frobenius/spectral norm.
Applying Theorem \ref{thm:main}, we immediately have the following result:
\begin{theorem}
There exists an absolute constant $c>0$ such that, if $\nu<c\cdot \sigma_{k+1}(\mat A)$ for some $1\leq k<n$, then with probability at least 0.8 we have that
\begin{equation}
\|\widehat{\mat A}_k-\mat A\|_F \leq \left(1+O\left(\sqrt{\frac{\nu}{\sigma_{k+1}(\mat A)}}\right)\right)\|\mat A-\mat A_k\|_F + O(\sqrt{k}\nu).
\label{eq:denoising}
\end{equation}
\label{thm:denoising}
\end{theorem}

Eq.~(\ref{eq:denoising}) can be understood from a classical bias-variance tradeoff perspective:
the first $(1+O(\sqrt{\nu/\sigma_{k+1}(\mat A)}))\|\mat A-\mat A_k\|_F$ acts as a bias term, which decreases as we increase cut-off rank $k$,
corresponding to a more complicated model;
on the other hand, the second $O(\sqrt{k}\nu)$ term acts as the (square root of) variance, which does not depend on the signal $\mat A$ and increases with $k$.



\subsection{Low-rank estimation of high-dimensional covariance}

Suppose $\mat A$ is an $n\times n$ PSD matrix and $X_1,\cdots,X_N$ are i.i.d.~samples drawn from the multivariate Gaussian distribution $\mathcal N_n(\vct 0,\mat A)$.
The question is to estimate $\mat A$ from samples $X_1,\cdots,X_N$.
A common estimator is the \emph{sample covariance} $\widehat{\mat A}=\frac{1}{N}\sum_{i=1}^N{X_iX_i^\top}$.
While in low-dimensional regimes (i.e., $n$ fixed and $N\to\infty$) the asymptotic efficiency of $\widehat{\mat A}$ is obvious (cf. \citep{van2000asymptotic}),
its statistical power in high-dimensional regimes where $n$ and $N$ are comparable are highly non-trivial.
Below we cite results by \citet{bunea2015sample} for estimation error $\|\widehat{\mat A}-\mat A\|_{\xi}$, $\xi=2/F$ when $n$ is not too large compared to $N$:
\begin{lemma}[\citet{bunea2015sample}]
Suppose $n=O(N^\beta)$ for some $\beta\geq 0$ and let $r_e(\mat A)=\tr(\mat A)/\|\mat A\|_2$ denote the \emph{effective rank} of the covariance $\mat A$.
Then the sample covariance $\widehat{\mat A}=\frac{1}{N}\sum_{i=1}^N{X_iX_i^\top}$ satisfies
\begin{equation}
\|\widehat{\mat A}-\mat A\|_F = O_\mP\left(\|\mat A\|_2 r_e(\mat A)\sqrt{\frac{\log N}{N}}\right)
\label{eq:bunea-frob}
\end{equation}
and
\begin{equation}
\|\widehat{\mat A}-\mat A\|_2 = O_\mP\left(\|\mat A\|_2\max\left\{\sqrt{\frac{r_e(\mat A)\log(Nn)}{N}},\frac{r_e(\mat A)\log(Nn)}{N}\right\}\right).
\label{eq:bunea-spec}
\end{equation}
\end{lemma}

Let $\widehat{\mat A}_k$ be the best rank-$k$ approximation of $\widehat{\mat A}$ in Frobenius/spectral norm.
Applying Theorem \ref{thm:main} and \ref{thm:main_with_gap} together with Eq.~(\ref{eq:bunea-spec}), we immediately arrive at the following theorem.
\begin{theorem}
Fix $\epsilon\in(0,1/4]$ and $1\leq k<n$. 
Recall that $r_e(\mat A)=\tr(\mat A)/\|\mat A\|_2$ and $\gamma_k(\mat A)=\sigma_1(\mat A)/\sigma_{k+1}(\mat A)$.
There exists a universal constant $c>0$ such that,
if
$$
\frac{r_e(\mat A)\max\{\epsilon^{-4},k^2\}\gamma_k(\mat A)^2\log(N)}{N}\leq c
$$
then with probability at least 0.8,
$$
\|\widehat{\mat A}_k-\mat A\|_F \leq \left(1+O(\epsilon)\right)\|\mat A-\mat A_k\|_F
$$
and if
$$
\frac{r_e(\mat A)k\|\mat A\|_2^2\log(N)}{N\epsilon^2 \left(\sigma_{k}\left(\mat A\right)-\sigma_{k+1}\left(\mat A\right)\right)^2}\leq c
$$
then with probability at least 0.8,
$$
\|\widehat{\mat A}_k-\mat A\|_F \leq \|\mat A-\mat A_k\|_F + \epsilon\left(\sigma_k\left(\mat A\right)-\sigma_{k+1}\left(\mat A\right)\right).
$$
\label{thm:covariance}
\end{theorem}

Theorem \ref{thm:covariance} shows that it is possible to obtain a reasonable Frobenius-norm approximation of $\widehat{\mat A}$
by truncated SVD in the asymptotic regime of $N=\Omega(r_e(\mat A)\mathrm{poly}(k)\log N)$, 
which is much more flexible than Eq.~(\ref{eq:bunea-frob}) that requires $N=\Omega(r_e(\mat A)^2\log N)$.

\section{Proof Sketch of Theorem~\ref{thm:main}}
\label{sec:proof_sketch}
In this section we give a proof sketch of Theorem~\ref{thm:main}.
The proof of Theorem~\ref{thm:main_with_gap} is similar and less challenging so we defer it to appendix.
We defer proofs of technical lemmas to Section~\ref{sec:proof}.

Because both $\widehat{\mat A}_k$ and $\mat A_k$ are low-rank, $\|\widehat{\mat A}_k-\mat A_k\|_F$ is upper bounded by 
an $O(\sqrt{k})$ factor of $\|\widehat{\mat A}_k-\mat A_k\|_2$.
From the condition that $\|\widehat{\mat A}-\mat A\|_2\leq\delta$, a straightforward approach to upper bound $\|\widehat{\mat A}_k-\mat A_k\|_2$
is to consider the decomposition $\|\widehat{\mat A}_k-\mat A_k\|_2 \leq \|\widehat{\mat A}-\mat A\|_2 + 2\|\mat U_k\mat U_k^\top-\widehat{\mat U}_k\widehat{\mat U}_k^\top\|_2\|\widehat{\mat A}_k\|_2$,
where $\mat U_k\mat U_k^\top$ and $\widehat{\mat U}_k\widehat{\mat U}_k^\top$ are projection operators onto the top-$k$ eigenspaces of $\mat A$ and $\widehat{\mat A}$, respectively.
Such a naive approach, however, has two major disadvantages.
First, the upper bound depends on $\|\widehat{\mat A}_k\|_2$, which is additive and may be much larger than $\|\widehat{\mat A}-\mat A\|_2$.
Perhaps more importantly, 
the quantity $\|\mat U_k\mat U_k^\top-\widehat{\mat U}_k\widehat{\mat U}_k^\top\|_2$ depends on the ``consecutive'' sepctral gap $(\sigma_k(\mat A)-\sigma_{k+1}(\mat A))$,
which could be very small for large matrices.

The key idea in the proof of Theorem \ref{thm:main} is to find an ``envelope'' $m_1\leq k\leq m_2$ in the spectrum of $\mat A$ surrounding $k$,
such that the eigenvalues within the envelope are relatively close.
Define
\begin{eqnarray*}
	m_1 &=& \argmax_{0\leq j\leq k}\{\sigma_j(\mat A)\geq (1+2\epsilon)\sigma_{k+1}(\mat A)\};\\
	m_2 &=& \argmax_{k\leq j\leq n}\{\sigma_j(\mat A) \geq \sigma_k(\mat A) - 2\epsilon\sigma_{k+1}(\mat A)\},
\end{eqnarray*}
where we let $\sigma_0\left(\mat A\right)=\infty$ for convenience.
Let $\mathcal U_m$, $\widehat{\mathcal{U}}_m$ be basis of the top $m$-dimensional linear subspaces of $\mat A$ and $\widehat{\mat A}$, respectively.
Also denote $\mathcal U_{n-m}$ and $\widehat{\mathcal U}_{n-m}$ as basis of the orthogonal complement of $\mathcal U_m$ and $\widehat{\mathcal U}_m$. 
By asymmetric Davis-Kahan inequality (Lemma~\ref{lem:davis-kahan}) and Wely's inequality we can obtain the following result.
\begin{lemma}
	If $\|\widehat{\mat A}-\mat A\|_2\leq\epsilon^2\sigma_{k+1}(\mat A)$ for $\epsilon\in(0,1)$ 
	then $\|\widehat{\mat U}_{n-k}^\top\mat U_{m_1}\|_2,\|\widehat{\mat U}_k^\top\mat U_{n-m_2}\|_2\leq\epsilon$.
	\label{lem:perturb}
\end{lemma}

Let $\mathcal U_{m_1:m_2}$ be the linear subspace of $\mat A$
associated with eigenvalues $\sigma_{m_1+1}(\mat A),\cdots,\sigma_{m_2}(\mat A)$.
Intuitively, we choose a $(k-m_1)$-dimensional linear subspace in $\mathcal U_{m_1:m_2}$ that is ``most aligned'' with the top-$k$ subspace $\widehat{\mathcal U}_k$
of $\widehat{\mat A}$.
Formally, define 
$$
\mathcal W = \argmax_{\dim(\mathcal W)=k-m_1, \mathcal W\in\mathcal U_{m_1:m_2}}\sigma_{k-m_1}\left(\mat W^\top\widehat{\mat U}_k\right).
$$
$\mat W$ is then a $d\times (k-m_1)$ matrix with orthonormal columns that corresponds to a basis of $\mathcal W$.
$\mathcal W$ is carefully constructed so that it is closely aligned with $\widehat{\mathcal U}_k$, yet still lies in $\mathcal U_k$.
In particular, Lemma \ref{lem:W} shows that $\sin\angle(\mathcal W,\widehat{\mathcal U}_k)=\|\widehat{\mat U}_{n-k}^\top\mat W\|_2$ is upper bounded by $\epsilon$.
\begin{lemma}
	If $\|\widehat{\mat A}-\mat A\|_2\leq\epsilon^2\sigma_{k+1}(\mat A)$ for $\epsilon\in(0,1)$ then $\|\widehat{\mat U}_{n-k}^\top\mat W\|_2\leq\epsilon$.
	\label{lem:W}
\end{lemma}
Now define
$$
\widetilde{\mat A} = \mat A_{m_1} + \mat W\mat W^\top\mat A\mat W\mat W^\top.
$$
We use $\widetilde{\mat A}$ as the ``reference matrix" because we can decompose $\|\widehat{\mat A}_k-\mat A\|_F$ as 
\begin{equation}
\|\widehat{\mat A}_k-\mat A\|_F \leq \|\mat A-\widetilde{\mat A}\|_F + \|\widehat{\mat A}_k-\widetilde{\mat A}\|_F \leq \|\mat A-\widetilde{\mat A}\|_F + \sqrt{2k}\|\widehat{\mat A}_k-\widetilde{\mat A}\|_2 \label{eq:decomp}
\end{equation}
and bound each term on the right hand side separately.
Here the last inequality holds because both $\widehat{\mat A}_k$ and $\widetilde{\mat A}$ have rank at most $k$.
The following lemma bounds the first term.
\begin{lemma}
	If $\|\widehat{\mat A}-\mat A\|_2\leq\epsilon^2\sigma_{k+1}(\mat A)^2$ for $\epsilon\in(0,1/4]$ then $\|\mat A-\widetilde{\mat A}\|_F\leq (1+32\epsilon)\|\mat A-\mat A_k\|_F$.
	\label{lem:decomp1}
\end{lemma}
The proof of this lemma relies Pythagorean theorem and Poincar\'{e} separation theorem.
	Let $\mathcal U_{m_1:m_2}$ be the $(m_2-m_1)$-dimensional linear subspace such that $\mathcal U_{m_2}=\mathcal U_{m_1}\oplus\mathcal U_{m_1:m_2}$.
	Define $\mat A_{m_1:m_2} = \mat U_{m_1:m_2}\mat\Sigma_{m_1:m_2}\mat U_{m_1:m_2}^\top$,
	where $\mat\Sigma_{m_1:m_2}=\diag(\sigma_{m_1+1}(\mat A),\cdots,\sigma_{m_2}(\mat A))$
	and $\mat U_{m_1:m_2}$ is an orthonormal basis associated with $\mathcal U_{m_1:m_2}$.
	Applying Pythagorean theorem (Lemma~\ref{lem:pythagorean}), we can decompose
	\begin{align*}
	\|\mat A-\widetilde{\mat A}\|_F^2 
	= \|\mat A-\mat A_{m_2}\|_F^2 + \|\mat A_{m_1:m_2}\|_F^2 - \|\mat W\mat W^\top\mat A_{m_1:m_2}\mat W\mat W^\top\|_F^2.
	\end{align*}
	Applying Poincar\'{e} separation theorem (Lemma \ref{lem:poincare}) where $\mat X = \mat \Sigma_{m_1:m_2}$ and $\mat P = \mat U^\top_{m_1:m2}\mat W$,
	we have $\|\mat W^\top\mat A_{m_1:m_2}\mat W\|_F^2 \geq \sum_{j=m_2-k+1}^{m_2-m_1}{\sigma_j(\mat A_{m_1:m_2})^2} = \sum_{j=m_1+m_2-k+1}^{m_2}{\sigma_j(\mat A)^2}$.
With some routine algebra we can prove Lemma~\ref{lem:decomp1}.

To bound the second term of Eq.~\eqref{eq:decomp} we use the following lemma.
\begin{lemma}
	If $\|\widehat{\mat A}-\mat A\|_2\leq\epsilon^2\sigma_{k+1}(\mat A)$ for $\epsilon\in(0,1/4]$ then $\|\widehat{\mat A}_k-\widetilde{\mat A}\|_2\leq 102\epsilon^2\|\mat A-\mat A_k\|_2$.
	\label{lem:decomp2}
\end{lemma}
The proof of Lemma~\ref{lem:decomp2} relies on the low-rankness of $\widehat{\mat A}_k$ and $\widetilde{\mat A}$.
	Recall the definition that $\widetilde{\mathcal U}=\Range(\widetilde{\mat A})$ and $\widetilde{\mathcal U}_\perp=\Null(\widetilde{\mat A})$.
	Consider $\|\vct v\|_2=1$ such that $\vct v^\top(\widehat{\mat A}_k-\widetilde{\mat A})\vct v = \|\widehat{\mat A}_k-\widetilde{\mat A}\|_2$.
	Because $\vct v$ maximizes $\vct v^\top(\widehat{\mat A}_k-\widetilde{\mat A})\vct v$ over all unit-length vectors,
	it must lie in the range of $\left(\widehat{\mat A}_k-\widetilde{\mat A}\right)$ because otherwise the component outside the range will not contribute.
	Therefore, we can choose $\vct v$ that $\vct v=\vct v_1+\vct v_2$ where $\vct v_1\in \Range(\widehat{\mat A}_k)=\widehat{\mathcal U}_k$ and $\vct v_2\in\Range(\widetilde{\mat A})=\widetilde{\mathcal U}$.
	Subsequently, we have that
	\begin{eqnarray}
	\vct v &=& \widehat{\mat U}_k\widehat{\mat U}_k^\top\vct v + \widetilde{\mat U}\widetilde{\mat U}^\top\widehat{\mat U}_{n-k}\widehat{\mat U}_{n-k}^\top\vct v\label{eq:v1}\\
	&=& \widetilde{\mat U}\widetilde{\mat U}^\top\vct v + \widehat{\mat U}_k\widehat{\mat U}_k^\top\widetilde{\mat U}_\perp\widetilde{\mat U}_\perp^\top\vct v.\label{eq:v2}
	\end{eqnarray}
	
	Consider the following decomposition:
	$$
	\left|\vct v^\top(\widehat{\mat A}_k-\widetilde{\mat A})\vct v\right| \leq \left|\vct v^\top(\widehat{\mat A}-\mat A)\vct v\right| + \left|\vct v^\top(\widehat{\mat A}_k-\widehat{\mat A})\vct v\right| + \left|\vct v^\top(\mat A-\widetilde{\mat A})\vct v\right|.
	$$
	The first term $|\vct v^\top(\widehat{\mat A}-\mat A)\vct v|$ is trivially upper bounded by $\|\widehat{\mat A}-\mat A\|_2\leq\epsilon^2\sigma_{k+1}(\mat A)$.
	The second and the third term can be bounded by Wely's inequality (Lemma~\ref{lem:weyl}) and basic properties of $\widetilde{\mat A}$ (Lemma~\ref{lem:tildeA-properties}).
	See Section~\ref{sec:proof} for details.

\section{Discussion}

We mention two potential directions to further extend results of this paper.

\subsection{Model selection for general high-rank matrices}

The validity of Theorem \ref{thm:main} depends on the condition $\|\widehat{\mat A}-\mat A\|_2\leq\epsilon^2\sigma_{k+1}(\mat A)$,
which could be hard to verify if $\sigma_{k+1}(\mat A)$ is unknown and difficult to estimate.
Furthermore, for general high-rank matrices,
the \emph{model selection} problem of determining an appropriate (or even optimal) cut-off rank $k$ requires knowledge of the distribution
of the entire spectrum of an unknown data matrix, which is even more challenging to obtain.

One potential approach is to impose a parametric pattern of decay of the eigenvalues (e.g., polynomial and exponential decay),
and to estimate a small set of parameters (e.g., degree of polynomial) from the noisy observations $\widehat{\mat A}$.
Afterwards, the optimal cut-off rank $k$ could be determined by a theoretical analysis, similar to the examples in Corollaries \ref{cor:power-law} and \ref{cor:exponential-decay}.
Another possibility is to use repeated sampling techniques such as boostrap in a stochastic problem (e.g., matrix de-noising) to 
estimate the ``bias'' term $\|\mat A-\mat A_k\|_F$ for different $k$, as the variance term $\sqrt{k}\nu$ is known or easy to estimate.

\subsection{Minimax rates for polynomial spectral decay}

Consider the class of PSD matrices whose eigenvalues follow a polynomial (power-law) decay:
$\Theta(\beta,n)=\{\mat A\in\mathbb R^{n\times n}: \mat A\succ 0, \sigma_j(\mat A)=j^{-\beta}\}$.
We are interested in the following minimax rates for completing or de-noising matrices in $\Theta(\beta,n)$:
\begin{question}[Completion of $\Theta(\beta,n)$]
Fix $n\in\mathbb N$, $p\in(0,1)$ and define $N=pn^2$.
For $\mat M\in\Theta(\beta,n)$, let $\widehat{\mat A}_{ij}=\mat M_{ij}$ with probability $p$ and $\widehat{\mat A}_{ij}= 0$ with probability $1-p$.
Also let $\Lambda(\mu_0,n)=\{\mat M\in\mathbb R^{n\times n}: n\|\mat M\|_{\max}\leq\mu_0\|\mat M\|_F\}$ be the class of all non-spiky matrices.
Determine 
$$
R_1(\mu_0,\beta,n,N) := \inf_{\widehat{\mat A}\mapsto\widehat{\mat M}}\sup_{\mat M\in\Theta(\beta,n)\cap\Lambda(\mu_0,n)} \mathbb E\|\widehat{\mat M}-\mat M\|_F^2.
$$
\end{question}
\begin{question}[De-noising of $\Theta(\beta,n)$]
Fix $n\in\mathbb N$, $\nu>0$ and let $\widehat{\mat A}=\mat M+\nu/\sqrt{n}\mat Z$, where $\mat Z$ is a symmetric matrices with i.i.d.~standard Normal random variables
on its upper triangle. Determine
$$
R_2(\nu,\beta,n) := \inf_{\widehat{\mat A}\mapsto\widehat{\mat M}}\sup_{\mat M\in\Theta(\beta,n)} \mathbb E\|\widehat{\mat M}-\mat M\|_F^2.
$$
\end{question}

Compared to existing settings on matrix completion and de-noising, we believe $\Theta(\beta,n)$ is a more natural matrix class 
which allows for general high-rank matrices, but also imposes sufficient spectral decay conditions so that spectrum truncation algorithms result in significant benefits.
Based on Corollary \ref{cor:power-law} and its matching lower bounds for a larger $\ell_p$ class \citep{negahban2012restricted}, we make the following conjecture:
\begin{conjecture}
For $\beta>1/2$ and $\nu$ not too small, we conjecture that
$$
R_1(\mu_0,\beta,n,N) \asymp C(\mu_0)\cdot \left[\frac{n}{N}\right]^{\frac{2\beta-1}{2\beta}} \;\;\;\;\;\;\text{and}\;\;\;\;\;\;
R_2(\nu,\beta,n) \asymp \left[\nu^2\right]^{\frac{2\beta-1}{2\beta}},
$$
where $C(\mu_0)>0$ is a constant that depends only on $\mu_0$.
\end{conjecture}

\section{Acknowledgements}
S.S.D. was supported by ARPA-E Terra program. 
Y.W. and A.S. were supported by the NSF CAREER grant IIS-1252412.

\bibliography{simonduref}
\bibliographystyle{plainnat}

\newpage
\appendix
\section{Proofs of Theorem \ref{thm:main} and Theorem~\ref{thm:main_with_gap}}\label{sec:proof}

The key idea in the proof of Theorem \ref{thm:main} is to find an ``envelope'' $m_1\leq k\leq m_2$ in the spectrum of $\mat A$ surrounding $k$,
such that the eigenvalues within the envelope are relatively close.
Define
\begin{eqnarray*}
m_1 &=& \argmax_{0\leq j\leq k}\{\sigma_j(\mat A)\geq (1+2\epsilon)\sigma_{k+1}(\mat A)\};\\
m_2 &=& \argmax_{k\leq j\leq n}\{\sigma_j(\mat A) \geq \sigma_k(\mat A) - 2\epsilon\sigma_{k+1}(\mat A)\},
\end{eqnarray*}
where we let $\sigma_0\left(\mat A\right)=\infty$ for convenience.
Let $\mathcal U_m$, $\widehat{\mathcal{U}}_m$ be basis of the top $m$-dimensional linear subspaces of $\mat A$ and $\widehat{\mat A}$, respectively.
Also denote $\mathcal U_{n-m}$ and $\widehat{\mathcal U}_{n-m}$ as basis of the orthogonal complement of $\mathcal U_m$ and $\widehat{\mathcal U}_m$.
\begin{lemma}
If $\|\widehat{\mat A}-\mat A\|_2\leq\epsilon^2\sigma_{k+1}(\mat A)$ for $\epsilon\in(0,1)$ 
 then $\|\widehat{\mat U}_{n-k}^\top\mat U_{m_1}\|_2,\|\widehat{\mat U}_k^\top\mat U_{n-m_2}\|_2\leq\epsilon$.
\end{lemma}
\begin{proof}
We apply an asymmetric version of Davis-Kahan inequality (Lemma \ref{lem:davis-kahan}), with $\mat X=\mat A$, $\mat Y=\widehat{\mat A}$, $i=m_1$ and $j=k$.
By Weyl's inequality, we know that $\sigma_{k+1}(\widehat{\mat A})\leq \sigma_{k+1}(\mat A)+\|\widehat{\mat A}-\mat A\|_2
\leq (1+\epsilon^2)\sigma_{k+1}(\mat A)\leq (1+\epsilon)\sigma_{k+1}(\mat A)$.
Subsequently, $\|\widehat{\mat U}_{n-k}^\top\mat U_{m_1}\|_2\leq \frac{\epsilon^2\sigma_{k+1}(\mat A)}{\sigma_{m_1}(\mat A)-(1+\epsilon)\sigma_{k+1}(\mat A)}\leq\epsilon$.
Similarly, applying Lemma \ref{lem:davis-kahan} with $\mat X=\widehat{\mat A}$, ${\mat Y}=\mat A$, $i=k$ and $j=m_2$ we have that
$\|\widehat{\mat U}_k^\top\mat U_{n-m_2}\|_2\leq\epsilon$.
\end{proof}

Let $\mathcal U_{m_1:m_2}$ be the linear subspace of $\mat A$
associated with eigenvalues $\sigma_{m_1+1}(\mat A),\cdots,\sigma_{m_2}(\mat A)$.
Intuitively, we choose a $(k-m_1)$-dimensional linear subspace in $\mathcal U_{m_1:m_2}$ that is ``most aligned'' with the top-$k$ subspace $\widehat{\mathcal U}_k$
of $\widehat{\mat A}$.
Formally, define 
$$
\mathcal W = \argmax_{\dim(\mathcal W)=k-m_1, \mathcal W\in\mathcal U_{m_1:m_2}}\sigma_{k-m_1}\left(\mat W^\top\widehat{\mat U}_k\right).
$$
$\mat W$ is then a $d\times (k-m_1)$ matrix with orthonormal columns that corresponds to a basis of $\mathcal W$.
$\mathcal W$ is carefully constructed so that it is closely aligned with $\widehat{\mathcal U}_k$, yet still lies in $\mathcal U_k$.
In particular, Lemma \ref{lem:W} shows that $\sin\angle(\mathcal W,\widehat{\mathcal U}_k)=\|\widehat{\mat U}_{n-k}^\top\mat W\|_2$ is upper bounded by $\epsilon$.
\begin{lemma}
If $\|\widehat{\mat A}-\mat A\|_2\leq\epsilon^2\sigma_{k+1}(\mat A)$ for $\epsilon\in(0,1)$ then $\|\widehat{\mat U}_{n-k}^\top\mat W\|_2\leq\epsilon$.
\end{lemma}
\begin{proof}
First note that $\|\widehat{\mat U}_{n-k}^\top\mat W\|_2 \leq \sqrt{1-\sigma_{k-m_1}(\widehat{\mat U}_k^\top\mat W)^2}$ because
\begin{align*}
\|\widehat{\mat U}_{n-k}^\top\mat W\|_2^2
&= \sup_{\|\vct x\|_2=1}\|\widehat{\mat U}_{n-k}^\top\mat W\vct x\|_2^2
= \sup_{\|\vct x\|_2=1}\left\{\|\mat W\vct x\|_2^2 - \|\widehat{\mat U}_k^\top\mat W\vct x\|_2^2\right\}\\
&\leq \sup_{\|\vct x\|_2=1}\|\mat W\vct x\|_2^2 - \inf_{\|\vct x\|_2=1}\|\widehat{\mat U}_k^\top\mat W\vct x\|_2^2
= 1-\sigma_{k-m_1}(\widehat{\mat U}_k^\top\mat W)^2.
\end{align*}
Subsequently, it suffices to prove that $\sigma_{k-m_1}(\widehat{\mat U}_k^\top\mat W)\geq\sqrt{1-\epsilon^2}$.
By Weyl's monotonicity theorem (Lemma \ref{lem:weyl}), we have that
$$
\sigma_k(\widehat{\mat U}_k^\top\mat U_{m_2}) \leq \sigma_{m_1+1}(\widehat{\mat U}_k^\top\mat U_{m_1}) + \sigma_{k-m_1}(\widehat{\mat U}_k^\top\mat U_{m_1:m_2}).
$$
In addition, $\sigma_{m_1+1}(\widehat{\mat U}_k^\top\mat U_{m_1})=0$ because $\rank(\widehat{\mat U}_k^\top\mat U_{m_1})\leq m_1$
and $\sigma_{k-m_1}(\widehat{\mat U}_k^\top\mat U_{m_1:m_2})=\sigma_{k-m_1}(\widehat{\mat U}_k^\top\mat W)$ because of the definition of $\mat W$.
Subsequently,
\begin{align*}
\sigma_{k-m_1}(\widehat{\mat U}_k^\top\mat W)^2
&\geq \sigma_k(\widehat{\mat U}_k^\top\mat U_{m_2})^2
= \inf_{\|\vct x\|_2=1}{\|\mat U_{m_2}^\top\widehat{\mat U}_k\vct x\|_2^2}
= \inf_{\|\vct x\|_2=1}\left\{\|\widehat{\mat U}_k\vec x\|_2^2 - \|\mat U_{n-m_2}^\top\widehat{\mat U}_{k}^\top\vct x\|_2^2\right\}\\
&\geq \inf_{\|\vct x\|_2=1}\left\{\| \widehat{\mat U}_k\vct x\|_2^2\right\} - \sup_{\|\vct x\|_2=1}\left\{\|\mat U_{n-m_2}^\top \widehat{\mat U}_{k}\vct x\|_2^2\right\}
\geq 1-\epsilon^2.
\end{align*}
Here in the last inequality we invoke Lemma \ref{lem:perturb}.
The proof is then complete.
\end{proof}

Define
$$
\widetilde{\mat A} = \mat A_{m_1} + \mat W\mat W^\top\mat A\mat W\mat W^\top.
$$
The following lemma lists some of the properties of $\widetilde{\mat A}$.
\begin{lemma}
It holds that
\begin{enumerate}
\item $\dim(\Range(\widetilde{\mat A})) = k$ and $\dim(\Range(\mat W))=k-m_1$;
\item $\mathcal U_{m_1}\subseteq \Range(\widetilde{\mat A})\subseteq\mathcal U_{m_2}$ and $\Range(\widetilde{\mat A}-\mat A_{m_1})\subseteq\mathcal U_{m_1:m_2}$,
where $\mathcal U_{m_2}=\mathcal U_{m_1}\oplus\mathcal U_{m_1:m_2}$.
\item $\|\widehat{\mat U}_k^\top\widetilde{\mat U}_\perp\|_2, \|\widetilde{\mat U}^\top\widehat{\mat U}_{n-k}\|_2\leq 2\epsilon$,
where $\widetilde{\mat U}$ and $\widetilde{\mat U}_\perp$ are orthonormal basis of $\Range(\widetilde{\mat A})$ and $\Null(\widetilde{\mat A})$, respectively.
\end{enumerate}
\label{lem:tildeA-properties}
\end{lemma}
\begin{proof}
Properties 1 and 2 are obviously true by the definition of $\mathcal W$ and $\widetilde{\mat A}$.
For property 3, note that both $\|\widehat{\mat U}_k^\top\widetilde{\mat U}_\perp\|_2$ and $\|\widetilde{\mat U}^\top\widehat{\mat U}_{n-k}\|_2$
are equal to $\sin\angle(\widetilde{\mathcal U}, \widehat{\mathcal U}_k)$.
Hence it suffices to show that $\|\widehat{\mat U}_{n-k}^\top\widetilde{\mat U}\|_2\leq 2\epsilon$.
Invoking Lemmas \ref{lem:perturb} and \ref{lem:W} we have that
$\|\widehat{\mat U}_{n-k}^\top\widetilde{\mat U}\|_2 \leq \|\widehat{\mat U}_{n-k}^\top\mat U_{m_1}\|_2 + \|\widehat{\mat U}_{n-k}^\top\mat W\|_2\leq\epsilon+\epsilon=2\epsilon$.
\end{proof}

Decompose $\|\widehat{\mat A}_k-\mat A\|_F$ as 
\begin{equation}
\|\widehat{\mat A}_k-\mat A\|_F \leq \|\mat A-\widetilde{\mat A}\|_F + \|\widehat{\mat A}_k-\widetilde{\mat A}\|_F \leq \|\mat A-\widetilde{\mat A}\|_F + \sqrt{2k}\|\widehat{\mat A}_k-\widetilde{\mat A}\|_2.
\end{equation}
Here the last inequality holds because both $\widehat{\mat A}_k$ and $\widetilde{\mat A}$ have rank at most $k$.
Lemmas \ref{lem:decomp1} and \ref{lem:decomp2} give separate upper bounds for $\|\mat A-\widetilde{\mat A}\|_F$ and $\|\widehat{\mat A}_k-\widetilde{\mat A}\|_2$.

\begin{lemma}
If $\|\widehat{\mat A}-\mat A\|_2\leq\epsilon^2\sigma_{k+1}(\mat A)^2$ for $\epsilon\in(0,1/4]$ then $\|\mat A-\widetilde{\mat A}\|_F\leq (1+32\epsilon)\|\mat A-\mat A_k\|_F$.
\end{lemma}
\begin{proof}
Let $\mathcal U_{m_1:m_2}$ be the $(m_2-m_1)$-dimensional linear subspace such that $\mathcal U_{m_2}=\mathcal U_{m_1}\oplus\mathcal U_{m_1:m_2}$.
Define $\mat A_{m_1:m_2} = \mat U_{m_1:m_2}\mat\Sigma_{m_1:m_2}\mat U_{m_1:m_2}^\top$,
where $\mat\Sigma_{m_1:m_2}=\diag(\sigma_{m_1+1}(\mat A),\cdots,\sigma_{m_2}(\mat A))$
and $\mat U_{m_1:m_2}$ is an orthonormal basis associated with $\mathcal U_{m_1:m_2}$.
We then have
\begin{align*}
\|\mat A-\widetilde{\mat A}\|_F^2 
&= \|\mat A_{n-m_1}-\mat W\mat W^\top\mat A\mat W\mat W^\top\|_F^2\\
&\overset{(a)}{=} \|\mat A_{n-m_2}\|_F^2 + \|\mat A_{m_1:m_2}-\mat W\mat W^\top\mat A\mat W\mat W^\top\|_F^2\\
&\overset{(b)}{=} \|\mat A-\mat A_{m_2}\|_F^2 + \|\mat A_{m_1:m_2}-\mat W\mat W^\top\mat A_{m_1:m_2}\mat W\mat W^\top\|_F^2\\
&\overset{(c)}{=} \|\mat A-\mat A_{m_2}\|_F^2 + \|\mat A_{m_1:m_2}\|_F^2 - \|\mat W\mat W^\top\mat A_{m_1:m_2}\mat W\mat W^\top\|_F^2.
\end{align*}
Here in $(a)$ we apply $\Range(\widetilde{\mat A}-\mat A_{m_1})\subseteq\mathcal U_{m_1:m_2}$ and the Pythagorean theorem (Lemma \ref{lem:pythagorean}) with $\mat P=\mat U_{m_1:m_2}$, 
in $(b)$ we apply $\mathcal W\subseteq\mathcal U_{m_1:m_2}$,
and in $(c)$ we apply the Pythagorean theorem again with $\mat P=\mat W$.
Note that $\|\mat W\mat W^\top\mat A_{m_1:m_2}\mat W\mat W^\top\|_F^2
= \|\mat W^\top\mat A_{m_1:m_2}\mat W\|_F^2$.
Applying Poincar\'{e} separation theorem (Lemma \ref{lem:poincare}) where $\mat X = \mat \Sigma_{m_1:m_2}$ and $\mat P = \mat U^\top_{m_1:m2}\mat W$,
we have $\|\mat W^\top\mat A_{m_1:m_2}\mat W\|_F^2 \geq \sum_{j=m_2-k+1}^{m_2-m_1}{\sigma_j(\mat A_{m_1:m_2})^2} = \sum_{j=m_1+m_2-k+1}^{m_2}{\sigma_j(\mat A)^2}$.
Subsequently, 
\begin{align*}
\|\mat A-\widetilde{\mat A}\|_F^2 
&\leq \|\mat A-\mat A_{m_2}\|_F^2 + \sum_{j=m_1+1}^{m_1+m_2-k}{\sigma_j(\mat A)^2} \leq \|\mat A-\mat A_{m_2}\|_F^2 + (m_2-k)\sigma_{m_1+1}(\mat A)^2\\
&\overset{(a')}{\leq} \|\mat A-\mat A_{m_2}\|_F^2 + (m_2-k)(1+2\epsilon)^2\sigma_{k+1}(\mat A)^2\\
&\overset{(b')}{\leq} \|\mat A-\mat A_{m_2}\|_F^2 + (m_2-k)\left(\frac{1+2\epsilon}{1-2\epsilon}\right)^2\sigma_{m_2}(\mat A)^2\\
&\overset{(c')}{\leq} \|\mat A-\mat A_{m_2}\|_F^2 + (m_2-k)\sigma_{m_2}(\mat A)^2 + 32(m_2-k)\epsilon\sigma_{m_2}(\mat A)^2\\
&\overset{(d')}{\leq} (1+32\epsilon)\|\mat A-\mat A_k\|_F^2.
\end{align*}
Here in $(a')$ we apply the definition of $m_1$ that $\sigma_{m_1+1}\leq (1+2\epsilon)\sigma_{k+1}(\mat A)$,
in $(b')$ we apply the definition of $m_2$ that $\sigma_{m_2}(\mat A)\geq \sigma_k(\mat A)-2\epsilon\sigma_{k+1}(\mat A) \geq (1-2\epsilon)\sigma_{k+1}(\mat A)$,
and $(c')$ is due to the fact that $\left(\frac{1+2\epsilon}{1-2\epsilon}\right)^2\leq 1+32\epsilon$ for all $\epsilon\in(0,1/4]$.
Finally, $(d')$ holds because $(m_2-k)\sigma_{m_2}(\mat A)^2 \leq \sum_{j=k+1}^{m_2}{\sigma_j(\mat A)^2}$
and $\|\mat A-\mat A_k\|_F^2 = \|\mat A-\mat A_{m_2}\|_F^2 + \sum_{j=k+1}^{m_2}{\sigma_j(\mat A)^2}$.
\end{proof}

\begin{lemma}
If $\|\widehat{\mat A}-\mat A\|_2\leq\epsilon^2\sigma_{k+1}(\mat A)$ for $\epsilon\in(0,1/4]$ then $\|\widehat{\mat A}_k-\widetilde{\mat A}\|_2\leq 102\epsilon^2\|\mat A-\mat A_k\|_2$.
\end{lemma}
\begin{proof}
Recall the definition that $\widetilde{\mathcal U}=\Range(\widetilde{\mat A})$ and $\widetilde{\mathcal U}_\perp=\Null(\widetilde{\mat A})$.
Consider $\|\vct v\|_2=1$ such that $\vct v^\top(\widehat{\mat A}_k-\widetilde{\mat A})\vct v = \|\widehat{\mat A}_k-\widetilde{\mat A}\|_2$.
Because $\vct v$ maximizes $\vct v^\top(\widehat{\mat A}_k-\widetilde{\mat A})\vct v$ over all unit-length vectors,
it must lie in the range of $\left(\widehat{\mat A}_k-\widetilde{\mat A}\right)$ because otherwise the component outside the range will not contribute.
Therefore, we can choose $\vct v$ that $\vct v=\vct v_1+\vct v_2$ where $\vct v_1\in \Range(\widehat{\mat A}_k)=\widehat{\mathcal U}_k$ and $\vct v_2\in\Range(\widetilde{\mat A})=\widetilde{\mathcal U}$.
Subsequently, we have that
\begin{eqnarray}
\vct v &=& \widehat{\mat U}_k\widehat{\mat U}_k^\top\vct v + \widetilde{\mat U}\widetilde{\mat U}^\top\widehat{\mat U}_{n-k}\widehat{\mat U}_{n-k}^\top\vct v\\
&=& \widetilde{\mat U}\widetilde{\mat U}^\top\vct v + \widehat{\mat U}_k\widehat{\mat U}_k^\top\widetilde{\mat U}_\perp\widetilde{\mat U}_\perp^\top\vct v.
\end{eqnarray}

Consider the following decomposition:
$$
\left|\vct v^\top(\widehat{\mat A}_k-\widetilde{\mat A})\vct v\right| \leq \left|\vct v^\top(\widehat{\mat A}-\mat A)\vct v\right| + \left|\vct v^\top(\widehat{\mat A}_k-\widehat{\mat A})\vct v\right| + \left|\vct v^\top(\mat A-\widetilde{\mat A})\vct v\right|.
$$
The first term $|\vct v^\top(\widehat{\mat A}-\mat A)\vct v|$ is trivially upper bounded by $\|\widehat{\mat A}-\mat A\|_2\leq\epsilon^2\sigma_{k+1}(\mat A)$.
For the second term, we have
\begin{align*}
\left|\vct v^\top(\widehat{\mat A}_k-\widetilde{\mat A})\vct v\right|
&= \left|\vct v^\top\widehat{\mat U}_{n-k}\widehat{\mat\Sigma}_{n-k}\widehat{\mat U}_{n-k}^\top\vct v\right\|\\
&\overset{(a)}{=} \left|\vct v^\top\widehat{\mat U}_{n-k}\widehat{\mat U}_{n-k}^\top\widetilde{\mat U}\widetilde{\mat U}^\top\widehat{\mat U}_{n-k}\widehat{\mat\Sigma}_{n-k}\widehat{\mat U}_{n-k}^\top\widetilde{\mat U}\widetilde{\mat U}^\top\widehat{\mat U}_{n-k}\widehat{\mat U}_{n-k}^\top\vct v\right|\\
&\leq \left\|\widehat{\mat U}_{n-k}^\top\widetilde{\mat U}\right\|_2^4\left\|\widehat{\mat U}_{n-k}\right\|_2
\overset{(b)}{\leq} 16\epsilon^4\sigma_{k+1}(\widehat{\mat A})
\overset{(c)}{\leq} 16\epsilon^4(1+\epsilon^2)\sigma_{k+1}(\mat A).
\end{align*}
Here in $(a)$ we apply Eq.~(\ref{eq:v1}); in $(b)$ we apply Property 3 of Lemma \ref{lem:tildeA-properties}, and $(c)$ is due to Weyl's inequality (Lemma \ref{lem:weyl}) 
that $\sigma_{k+1}(\widehat{\mat A})\leq \sigma_{k+1}(\mat A) + \|\widehat{\mat A}-\mat A\|_2 \leq (1+\epsilon^2)\sigma_{k+1}(\mat A)$.

For the third term, note that $\widetilde{\mat A}=\widetilde{\mat U}\widetilde{\mat U}^\top\mat A\widetilde{\mat U}\widetilde{\mat U}^\top$
because $\Range(\widetilde{\mat A})\subseteq\mathcal U_{m_2}\subseteq\Range(\mat A)$ by Lemma \ref{lem:tildeA-properties}.
Subsequently,
$$
\mat A-\widetilde{\mat A}
= \underbrace{\widetilde{\mat U}_\perp\widetilde{\mat U}_\perp^\top\mat A\widetilde{\mat U}_\perp\widetilde{\mat U}_\perp^\top}_{\mat B_1}
+ \underbrace{\widetilde{\mat U}\widetilde{\mat U}^\top\mat A\widetilde{\mat U}_\perp\widetilde{\mat U}_\perp^\top}_{\mat B_2}
+ \underbrace{\widetilde{\mat U}_\perp\widetilde{\mat U}_\perp^\top\mat A\widetilde{\mat U}\widetilde{\mat U}^\top}_{\mat B_2^\top}.
$$
It then suffices to upper bound $|\vct v^\top\mat B_1\vct v|$ and $|\vct v^\top\mat B_2\vct v|$ separately.
For $\mat B_1$ we have
\begin{align*}
\left|\vct v^\top\mat B_1\vct v\right|
&\overset{(a')}{=} \left|\vct v^\top\widetilde{\mat U}_\perp\widetilde{\mat U}_\perp^\top\widehat{\mat U}_k\widehat{\mat U}_k^\top\widetilde{\mat U}_\perp\widetilde{\mat U}_\perp^\top\mat A\widetilde{\mat U}_\perp\widetilde{\mat U}_\perp^\top\widehat{\mat U}_k\widehat{\mat U}_k^\top\widetilde{\mat U}_\perp\widetilde{\mat U}_\perp^\top\vct v\right|\\
&\leq \left\|\widetilde{\mat U}_\perp^\top\widehat{\mat U}_k\right\|_2^4\left\|\widetilde{\mat U}_\perp^\top\mat A\widetilde{\mat U}_\perp\right\|_2\\
&\overset{(b')}{\leq} 16\epsilon^4\left\|\widetilde{\mat U}_\perp^\top\mat A\widetilde{\mat U}_\perp\right\|_2
\overset{(c')}{\leq}16\epsilon^4\sigma_{m_1+1}(\mat A)
\overset{(d')}{\leq}16 \epsilon^4(1+2\epsilon)\sigma_{k+1}(\mat A).
\end{align*}
Here in $(a')$ we apply Eq.~(\ref{eq:v2}); in $(b')$ we apply Property 3 of Lemma \ref{lem:tildeA-properties};
$(c')$ follows the property that $\widetilde{\mathcal U}_\perp\in\mathcal U_{n-m_1}$,
and finally $(d')$ follows from the definition of $m_1$ that $\sigma_{m_1+1}(\mat A)\leq (1+2\epsilon)\sigma_{k+1}(\mat A)$.

For $\mat B_2$, we have that
\begin{align*}
\left|\vct v^\top\mat B_2\vct v\right|
&{=} \left|\vct v^\top\widetilde{\mat U}\widetilde{\mat U}^\top\mat A\widetilde{\mat U}_\perp\widetilde{\mat U}_\perp^\top\widehat{\mat U}_k\widehat{\mat U}_k^\top\widetilde{\mat U}_\perp\widetilde{\mat U}_\perp^\top\vct v\right|\\
&\leq \left\|\mat A\widetilde{\mat U}_\perp\right\|_2\left\|\widetilde{\mat U}_\perp^\top\widehat{\mat U}_k\right\|_2^2 \leq \epsilon^2(1+8\epsilon)\sigma_{k+1}(\mat A).
\end{align*}
Combining all inequalities and noting that $\epsilon\in(0,1/4]$, we obtain
\begin{align*}
\|\widehat{\mat A}_k-\widetilde{\mat A}\|_2 &\leq \epsilon^2\sigma_{k+1}(\mat A) + 16\epsilon^4(1+2\epsilon+\epsilon^2)\sigma_{k+1}(\mat A) + 32\epsilon^2(1+8\epsilon)\sigma_{k+1}(\mat A)\\
&\leq 102\epsilon^2\sigma_{k+1}(\mat A).
\end{align*}

\end{proof}

\begin{proof}{\bf of Theorem~\ref{thm:main_with_gap}}
The proof of Theorem~\ref{thm:main_with_gap} is similar and even simpler than that of Theorem~\ref{thm:main}.
First observing that with the large spectral gap, $\widetilde{\mat A} = \mat A_k$.
Next we replace by replacing the assumption $\|\widehat{\mat A} - \mat A \|_2 \le \epsilon^2\sigma_{k+1}(\mat A)$ in Lemma~\ref{lem:decomp2} with $\|\widehat{\mat A} - \mat A \|_2 \le \epsilon\left(\sigma_{k}(\mat A) - \sigma_{k+1}(\mat A)\right)$ using the exactly the same arguments we have
\[
\|\widehat{\mat A}_k-\mat A_k\|_2 \le 102\epsilon\left(\sigma_{k}(\mat A) - \sigma_{k+1}(\mat A)\right).
\]
Therefore, we have
\[ 
\|\widehat{\mat A}_k-\mat A_k\|_F \le 102\sqrt{2k}\epsilon\left(\sigma_{k}(\mat A) - \sigma_{k+1}(\mat A)\right).
\]
Lastly, apply triangle inequality:\begin{align*}
\|\widehat{\mat A}_k - \mat A\|_F & \le \|\mat A - \mat A_k\|_F + \|\widehat{\mat A}_k-\mat A_k\|_F \\
&\le \|\mat A - \mat A_k\|_F +  102\sqrt{2k}\epsilon\left(\sigma_{k}(\mat A) - \sigma_{k+1}(\mat A)\right).
\end{align*}

\end{proof}

\section{Proof of corollaries}

\begin{proof}{\bf of Corollary \ref{cor:power-law}}.
We first verify the condition that $\delta\leq\epsilon^{2}\sigma_{k+1}(\mat A)$ for $\epsilon=1/4$ and the particular choice of $k$.
Because $k\leq\lfloor C_1\delta^{-1/\beta}\rfloor-1$, we have that $\sigma_{k+1}(\mat A)\geq (C_1\delta^{-1/\beta})^{-\beta}$.
By carefully chosen $C_1$ (depending on $\beta$) the inequality $\sigma_{k+1}(\mat A)\geq \delta/16$ holds.

If $k=n-1$ then by Theorem \ref{thm:main}, $\|\widehat{\mat A}_k-\mat A\|_F \leq O(\sqrt{n}\cdot n^{-\beta}) = O(n^{-\frac{2\beta-1}{2\beta}})$.
In the rest of the proof we assume $k=\lfloor C_1\delta^{-1/\beta}\rfloor-1$.
We then have
$$
\|\mat A-\mat A_k\|_F 
= \sqrt{\sum_{j=k+1}^n{\sigma_j(\mat A)^2}} = \sqrt{\sum_{j=k+1}^n{j^{-2\beta}}} 
\leq \sqrt{\int_k^{\infty}{x^{-2\beta}\ud x}} = \sqrt{\frac{k^{-(2\beta-1)}}{2\beta-1}} \leq C(\beta) \delta^{\frac{2\beta-1}{2\beta}}.
$$
Here $C(\beta)>0$ is a constant that only depends on $\beta$.
In addition, 
$$
\sqrt{k}\|\mat A-\mat A_k\|_2 \leq \sqrt{k}\cdot k^{-\beta} = k^{-(\beta-1/2)} \leq \tilde C(\beta)\delta^{\frac{2\beta}{2\beta-1}}.
$$
Applying Theorem \ref{thm:main} we complete the proof of Corollary \ref{cor:power-law}.
\end{proof}

\begin{proof}{\bf of Corollary \ref{cor:exponential-decay}}
We first verify the condition that $\delta\leq\epsilon^2\sigma_{k+1}(\mat A)$ for $\epsilon=1/4$ and the particular choice of $k$.
Because $k\leq\lfloor c^{-1}\log(1/\delta)-c^{-1}\log\log(1/\delta)\rfloor-1$, we have that $\sigma_{k+1}(\mat A)\geq \delta\log(1/\delta)$.
Hence, for $\delta\in(0,e^{-16})$ it holds that $\sigma_{k+1}(\mat A)\geq \delta/16$.

If $k=n-1$ then by Theorem \ref{thm:main}, $\|\widehat{\mat A}_k-\mat A\|_F \leq O(\sqrt{n}\cdot \exp\{-cn\})$.
In the rest of the proof we assume $k=\lfloor C_2\log(1/\delta)\rfloor-1$.
We then have
$$
\|\mat A-\mat A_k\|_F = \sqrt{\sum_{j=k+1}^n{\sigma_j(\mat A)^2}} = \sqrt{\sum_{j=k+1}^n{\exp\{-2cj\}}} \leq \sqrt{\frac{\exp\{-2ck\}}{1-e^{-2c}}}
\leq C(c)\delta\log(1/\delta),
$$
where $C(c)>0$ is a constant that only depends on $c$.
In addition,
$$
\sqrt{k}\|\mat A-\mat A_k\|_2 \leq \sqrt{k}\cdot \exp\{-ck\} \leq \delta\log(1/\delta)\cdot\sqrt{c^{-1}\log(1/\delta)}
\leq \tilde C(c)\delta\sqrt{\log(1/\delta)^3}.
$$
Applying Theorem \ref{thm:main} we complete the proof of Corollary \ref{cor:exponential-decay}.
\end{proof}

\section{Technical lemmas}

\begin{lemma}[Asymmetric Davis-Kahan inequality]
Fix $i\le j\leq n$ and suppose $\mat X,\mat Y$ are symmetric $n\times n$ matrices,
with eigen-decomposition $\mat X=\mat P_i\mat\Lambda_i\mat P_i^\top + \mat P_{n-i}\mat\Lambda_{n-i}\mat P_{n-i}^\top$
and $\mat Y=\mat Q_j\mat\Xi_j\mat Q_j^\top + \mat Q_{n-j}\mat\Xi_{n-j}\mat Q_{n-j}^\top$.
If $\sigma_i(\mat X)>\sigma_{j+1}(\mat Y)$ then
$$
\|\mat Q_{n-j}^\top\mat P_i\|_2\leq \frac{\|\mat X-\mat Y\|_2}{\sigma_i(\mat X)-\sigma_{j+1}(\mat Y)}.
$$
\label{lem:davis-kahan}
\end{lemma}
\begin{proof}
Consider
$$
\left\|{\mat Q}_{n-j}^\top({\mat X}-\mat Y)\mat P_{i}\right\|_2
= \left\|{\mat Q}_{n-j}^\top\mat P_{i}\mat\Lambda_{i} - {\mat\Xi}_{n-j}{\mat Q}_{n-j}^\top\mat P_{i}\right\|_2
\geq \left\|\mat Q_{n-j}^\top\mat P_i\right\|_2\left(\sigma_i(\mat X)-\sigma_{j+1}(\mat Y)\right).
$$
Because $\sigma_i(\mat X)>\sigma_{j+1}(\mat Y)$, we have that
$$
\left\|\mat Q_{n-j}^\top\mat P_i\right\|_2 \leq \frac{\|\mat Q_{n-j}^\top(\mat X-\mat Y)\mat P_i\|_2}{\sigma_i(\mat X)-\sigma_{j+1}(\mat Y)}
\leq \frac{\|\mat X-\mat Y\|_2}{\sigma_i(\mat X)-\sigma_{j+1}(\mat Y)}.
$$
\end{proof}

\begin{lemma}[Pythagorean theorem]
Fix $n\geq m$. Suppose $\mat X$ is a symmetric $n\times n$ matrix and $\mat P$ is an $n\times m$ matrix satisfying $\mat P^\top\mat P=\mat I$.
Then $\|\mat X\|_F^2 = \|\mat X-\mat P\mat P^\top\mat X\mat P\mat P^\top\|_F^2 + \|\mat P\mat P^\top\mat X\mat P\mat P^\top\|_F^2$.
\label{lem:pythagorean}
\end{lemma}
\begin{proof}
Expanding $\|\mat X\|_F^2$ we have that
\begin{align*}
\|\mat X\|_F^2 
&= \|(\mat X-\mat P\mat P^\top\mat X\mat P\mat P^\top)+\mat P\mat P^\top\mat X\mat P\mat P^\top\|_F^2\\
&= \|\mat X-\mat P\mat P^\top\mat X\mat P\mat P^\top\|_F^2 + \|\mat P\mat P^\top\mat X\mat P\mat P^\top\|_F^2 + 2\tr\left[(\mat X-\mat P\mat P^\top\mat X\mat P\mat P^\top)\mat P\mat P^\top\mat X\mat P\mat P^\top\right].
\end{align*}
It suffices to prove that the trace term is zero:
\begin{align*}
 \tr\left[(\mat X-\mat P\mat P^\top\mat X\mat P\mat P^\top)\mat P\mat P^\top\mat X\mat P\mat P^\top\right]
&= \tr\left(\mat X\mat P\mat P^\top\mat X\mat P\mat P^\top\right) - \tr\left(\mat P\mat P^\top\mat X\mat P\mat P^\top\mat P\mat P^\top\mat X\mat P\mat P^\top\right)\\
&\overset{(*)}{=} \tr\left(\mat P^\top\mat X\mat P\mat P^\top\mat X\mat P\right) - \tr\left(\mat P^\top\mat X\mat P\mat P^\top\mat X\mat P\right)\\
&= 0.
\end{align*}
Here $(*)$ is due to $\mat P^\top\mat P=\mat I$.
\end{proof}

\begin{lemma}[Poincar\'{e} separation theorem]
Fix $n\geq m$.
Suppose $\mat X$ is a symmetric $n\times n$ matrix, $\mat P$ is an $n\times m$ matrix that satisfies $\mat P^\top\mat P=\mat I$,
and $\mat Y=\mat P^\top\mat X\mat P$.
Let $\sigma_1(\mat X)\geq\cdots\geq\sigma_n(\mat X)$ and $\sigma_1(\mat Y)\geq\cdots\geq\sigma_m(\mat Y)$
be the eigenvalues of $\mat X$ and $\mat Y$ in descending order.
Then
$$
\sigma_i(\mat X) \geq \sigma_i(\mat Y) \geq \sigma_{n-m+i}(\mat X), \;\;\;\;\;\;i=1,\cdots,m.
$$
\label{lem:poincare}
\end{lemma}

\begin{lemma}[Weyl's monotonicity theorem]
Suppose $\mat X,\mat Y$ are $n\times n$ symmetric matrices, and let $\sigma_1(\mat X)\geq\cdots\geq\sigma_n(\mat X)$, $\sigma_1(\mat Y)\geq\cdots\geq\sigma_n(\mat Y)$
and $\sigma_1(\mat X+\mat Y)\geq\cdots\geq\sigma_n(\mat X+\mat Y)$ denote the eigenvalues of $\mat X,\mat Y$ and $\mat X+\mat Y$ in descending order. Then
$$
\sigma_{i+j-1}(\mat X+\mat Y)\leq \sigma_i(\mat X)+\sigma_j(\mat Y), \;\;\;\;\;\;1\leq i,j\leq n, i+j-1\leq n.
$$
In particular, setting $i=1$ one obtains the commonly used Weyl's inequality: $|\sigma_j(\mat X+\mat Y)-\sigma_j(\mat X)|\leq \|\mat Y\|_2$.
\label{lem:weyl}
\end{lemma}





\end{document}